\documentclass[9pt]{article}

\usepackage{spconf}
\usepackage{amsmath,amssymb,amsthm}
\usepackage{graphicx}
\usepackage{braket}
\usepackage{epstopdf}
\usepackage{enumitem}
\usepackage{subcaption}

\newcommand{\Dm}{\mathbf{D}}

\newcommand{\Id}{\mathbf{I}}
\newcommand{\Lm}{\mathbf{L}}
\newcommand{\Xm}{\mathbf{X}}
\newcommand{\Wm}{\mathbf{W}}

\newcommand{\xv}{\mathbf{x}}
\newcommand{\yv}{\mathbf{y}}
\newcommand{\zv}{\mathbf{z}}
\newcommand{\fv}{\mathbf{f}}
\newcommand{\sv}{\mathbf{s}}

\newcommand{\onev}{\mathbf{1}}

\newcommand{\E}[1]{\mathbb{E}\left\{{#1}\right\}}
\newcommand{\Var}[1]{\text{Var}\left\{{#1}\right\}}

\newcommand{\Order}[1]{O\left({#1}\right)}

\newcommand{\Sc}{\mathcal{S}}

\newtheorem{theorem}{Theorem}

\newtheorem{lemma}{Lemma}

\newtheorem{corollary}{Corollary}
\newtheorem{observation}{Observation}

\title{Asymptotic Justification of Bandlimited Interpolation of Graph signals for Semi-Supervised Learning}

\name{
Aamir Anis, 
Aly El Gamal,
Salman Avestimehr 
and Antonio Ortega
\thanks{This work was supported in part by NSF under grant CCF-1410009 and NSF Grant 1408639.}
}
\address{Department of Electrical Engineering\\
	University of Southern California, Los Angeles\\
	Email: \{aanis, aelgamal\}@usc.edu, avestimehr@ee.usc.edu, ortega@sipi.usc.edu}

\begin{document}
\ninept
\maketitle
\begin{abstract}
Graph-based methods play an important role in unsupervised and semi-supervised learning tasks by taking into account the underlying geometry of the data set. 
In this paper, we consider a statistical setting for semi-supervised learning and provide a formal justification of the recently introduced framework of bandlimited interpolation of graph signals.
Our analysis leads to the interpretation that, given enough labeled data, this method is very closely related to a constrained low density separation problem as the number of data points tends to infinity.
We demonstrate the practical utility of our results through simple experiments.
\end{abstract}
\begin{keywords}
Graph signal processing, semi-supervised learning, interpolation, asymptotics
\end{keywords}
%
%
%
\section{Introduction}
Recently, graph-based methods have been employed very successfully in solving the semi-supervised learning (SSL) problem~\cite{zhu_03, zhou_04, belkin_06}. 
The underlying approach involves constructing a geometric graph from the data set, where the nodes correspond to data points and the edge weights indicate similarities between them, generally computed as a function of their distance in the feature space.
These methods are particularly attractive as they allow one to introduce priors for smoothness, or local and global consistency in the data labels (see for example, the graph Laplacian regularizer $\fv^T \Lm \fv$ and its variations~\cite{zhu_03, zhou_04}). 
\par
An insightful way of justifying graph-based learning algorithms is to study their behavior on statistical data
in the large sample limit.
Several papers have analyzed the stochastic convergence of cuts on a similarity graph constructed from data points sampled from a probability distribution $p(\xv)$. 
As the sample size goes to infinity and for a specific graph construction scheme, the cut is shown to converge to a weighted volume of the boundary: $\int_{\partial \Sc} p^\alpha(\sv)d\sv$ for some $\alpha > 0$ that depends on the graph definition~\cite{maier_13}. 
These results serve as a justification for spectral clustering, since searching for the minimum cut on the similarity graph is equivalent to a low density separation problem in the asymptotic limit.
Similar arguments hold for SSL problems, where the regularizer $\fv^T \Lm \fv$ has been shown to converge to a weighted energy expression of the form: $\int \|\nabla f(\xv)\|^2 p^\alpha(\xv) d\xv$~\cite{hein_thesis_06}.
Using this expression as a penalty 
ensures that the predicted labels do not vary much in regions of high density.  
\par
More recently, SSL has also been viewed from a graph signal processing perspective, where class indicator vectors are considered as smooth signals defined on the similarity graph (see~\cite{shuman_13, sandryhaila_13, sandryhaila_14} for an overview on graph signal processing).
Specifically, in this setting, one incorporates smoothness in the indicator vectors by approximating them with bandlimited or lowpass signals with respect to the graph's Fourier basis. 
The advantage of such an approach lies in the fact that, by using the sampling theorem for graph signals~\cite{anis_14}, it is possible to state 
conditions that guarantee perfect prediction of the unknown labels.
Then, the task of learning simply translates to one of recovering a bandlimited graph signal from its known sample values~\cite{narang_icassp13,narang_globalsip13,gadde_14}.
We call this approach Bandlimited Interpolation of Graph signals (BIG). 
\par
However, using BIG for SSL does not have a very clear theoretical justification. 
Moreover, its connections with existing graph-based methods in SSL are not fully understood.
Specifically, one needs to consider the following questions: firstly, how does the interpolated class indicator signal compare to other indicator signals satisfying the label constraints? And secondly, how does the bandwidth of class indicator signals relate asymptotically to $p(\xv)$ in the statistical setting for SSL?
\par
The focus of this work is to provide a formal justification for BIG, and draw connections with existing methods. 
We answer the first question using the graph sampling theorem: given enough labeled data, the interpolated indicator signal has minimum bandwidth among all indicator signals that satisfy the label constraints.
We then show in a statistical setting that an estimate of the bandwidth for any indicator signal, on a specifically constructed graph, asymptotically matches the supremum value of the probability distribution over the corresponding decision boundary associated with the indicator, as the number of data points, and thus the graph size, goes to infinity.
The two results put together suggest an interpretation for the BIG
approach in SSL problems: given, enough labeled data, BIG learns 
%
%
a decision boundary that respects the labels and over which the maximum density of the data points is as low as possible, similar to other graph-based methods. 
In summary, we observe from our result and previous analyses of spectral clustering that asymptotically, there is a strong
link between the value of a cut and the bandwidth of its associated indicator signal.
Thus, the geometric properties desired of ``minimal cuts" in clustering translate to
those of ``minimal bandwidth" indicator signals for classification in the
presence of labels.
\section{Graph-based learning}\label{sec:graph}
%
%
We now introduce the problem setting considered in this paper. \vspace{0.1cm} \\
\textbf{Data Model:} We assume that the data set consists of $n$ random feature vectors $X = \{\Xm_1, \Xm_2, \dots, \Xm_n\}$ drawn \emph{independently} from some probability density function $p(\xv)$ on $\mathbb{R}^d$.
Let $\partial S$ be a smooth hypersurface that splits $\mathbb{R}^d$ into two disjoint parts $S$ and $S^c$
(multiclass problems can be modeled using the one-vs-all approach).
Further, let $X_S = X \cap S$ and $X_{S^c} = X \cap S^c$ be the set of points that land in $S$ and $S^c$ respectively. We denote the indicator vector for $X_S$ by $\onev_S \in \{0,1\}^n$: $\onev_S(i)$ equals $1$ if $\Xm_i \in X_S$ and $0$ otherwise. \vspace{0.1cm} \\
\textbf{Learning task:}
We consider the problem of semi-supervised learning, where the labels of a \emph{small} subset of data points $X_L \subset X$ are known and the task is to predict the labels of the unlabeled set $X_U = X \setminus X_L$. 
More precisely, we would like to obtain $\onev_S(U)$ from $X$ and $\onev_S(L)$, where $\onev_S(U) \in \{0,1\}^{|X_U|}$ and $\onev_S(L) \in \{0,1\}^{|X_L|}$ denote the membership, with respect to $X_S$, of the unlabeled and labeled sets of points respectively. \vspace{0.1cm} \\
\textbf{Graph model:} We construct a distance-based similarity graph with data points as nodes and edge weights given by the Gaussian kernel:
\begin{equation}\label{eq:graph}
w_{ij} = K_{\sigma^2}(\Xm_i,\Xm_j) = \frac{1}{(2\pi\sigma^2)^{d/2}} \exp{\left( -\frac{\|\Xm_i - \Xm_j\|^2}{2\sigma^2} \right)}
\end{equation}
Further, we assume $w_{ii} = 0$, i.e., the graph does not have self-loops. 
The adjacency matrix of the graph $\Wm$ is a symmetric matrix with elements $w_{ij}$, while the degree matrix is a diagonal matrix with elements $\Dm_{ii} = \sum_j w_{ij}$. 
We define the graph Laplacian as $\Lm = \frac{1}{n}(\Dm-\Wm)$. Normalization ensures the norm of $\Lm$ is stochastically bounded as $n$ increases.
%
%
%
\subsection{Spectral Clustering on Graphs}
Convergence of cuts has been studied before in the context of spectral clustering, where one tries to minimize the graph cut across two partitions of the nodes.
Note that the empirical value of the graph cut induced by the boundary $\partial S$ can be expressed in terms of the indicator vector $\onev_S$ for $S$ and the graph Laplacian as: 
\begin{equation} 
\textit{Cut}(S,S^c) = \sum_{i \in S, j \in S^c} w_{ij} = \onev_S^T \Lm \onev_S. 
\end{equation}
It has been shown in~\cite{maier_13} that the following convergence theorem (stated in a simple form) holds for hyperplanes $\partial S$ in $\mathbb{R}^d$:
\begin{theorem}
Under the conditions $\sigma \rightarrow 0$ and $n\sigma^{d+1} \rightarrow \infty$,
\begin{equation}
\frac{\sqrt{2\pi}}{n\sigma} \onev_S^T \Lm \onev_S \xrightarrow{p} \int_{\partial S} p^2(\sv)d\sv,
\end{equation}
where $d\sv$ ranges over all $(d-1)$-dimensional volume elements tangent to the hyperplane $\partial S$.
\end{theorem}
\noindent A similar result has been shown earlier for smooth hypersurfaces \cite{narayanan_06}. The condition $\sigma \rightarrow 0$ leads to a clear and well-defined limit on the right hand side. Intuitively, it enforces sparsity in the similarity matrix $\Wm$ by shrinking the neighborhood volume as the number of data points increases. As a result, one can ensure that the graph remains sparse even though the number of points goes to infinity.

The result above has significant implications for spectral clustering: With certain scaling, the empirical cut value converges to a weighted volume of the boundary, thus spectral clustering is a means of performing low density separation on a finite sample.
\par
\subsection{Graph Laplacian Regularization for SSL}
In SSL, one generally exploits the availability of labeled samples to reconstruct an unknown function $\fv$ as follows:
\begin{equation}\label{eq:ssl_formulation}
\text{Minimize } \fv^T \Lm \fv \text{ such that } \fv(L) = \onev_S(L).
\end{equation}
Note that $\fv$ is generally not restricted to be an indicator and is taken to be a smooth signal in $\mathbb{R}^n$. One particular  convergence result in this setting can be stated as follows~\cite{hein_thesis_06, nadler_09}:
\begin{theorem}
Under the conditions $\sigma \rightarrow 0$ and $n\sigma^{d} \rightarrow \infty$,
\begin{equation}\label{eq:beliefresult}
\frac{1}{n \sigma^2} \fv^T \Lm \fv \xrightarrow{p} C \int \| \nabla f(\xv) \|^2 p^2(\xv) d\xv,
\end{equation}
where for each $n$, $\fv$ is a vector representing the values of $f(\xv)$ at the $n$ sample points and $C$ is a constant factor independent of $n$ and $\sigma$.
\end{theorem}
\noindent Similar to the justification of spectral clustering, this result justifies the formulation in \eqref{eq:ssl_formulation} for SSL: Given label constraints, the predicted signal must vary little in regions of high density.
\subsection{Bandlimited Interpolation of Graph signals (BIG)}
The task in BIG is to recover a bandlimited signal closest to the indicator signal satisfying the label constraints. Let $\omega(\fv)$ denote the bandwidth of a signal $\fv$ and $PW_\omega(G)$ (Payley-Wiener space with cutoff frequency $\omega$~\cite{anis_14}) denote the set of $\omega$-bandlimited signals on the graph $G$, i.e., $PW_\omega(G) = \{ \fv \;|\; \omega(\fv) < \omega \}$.
Then, the BIG method essentially consists of
\begin{enumerate}[topsep=1pt,itemsep=0pt,leftmargin=15pt]
\item Estimating the cut-off frequency $\omega_L$ associated with the labeled set $X_L$ using the sampling theorem for graph signals~\cite{anis_14}.
\item Estimating the desired indicator vector $\onev_S$ from labels $\onev_S(L)$ by solving the following least-squares problem:
\begin{equation}
\fv_\text{LS} = \mathop{\text{arg min}}_\fv \; \| \fv(L) - \onev_S(L)\|^2 \quad \text{s.t.} \quad \fv \in PW_{\omega_L}(G).
\label{eq:least_squares}
\end{equation}
\end{enumerate}
%
%
This method has been considered earlier \cite{belkin_04}, albeit, with an arbitrary choice of $\omega_L$.
Note that if the original indicator $\onev_S$ is bandlimited with respect to the labeled set, (i.e., $\omega(\onev_S) < \omega_L$), then the estimate $\fv_\text{LS}$ in \eqref{eq:least_squares} is guaranteed to be equal to $\onev_S$ as a consequence of the sampling theorem. 
Moreover, 
in this case,
$\onev_S$ can also be perfectly estimated by
the solution of the following ``dual" problem:
\begin{equation}
\fv_\text{min} = \mathop{\text{arg min}}_\fv \; \omega(\fv) \quad \text{s.t.} \quad \fv(L) = \onev_S(L),
\label{eq:omega_min}
\end{equation}
These facts leads to the following insight regarding BIG for SSL:
\begin{observation}
If $\omega(\onev_S) < \omega_L$, then
\begin{enumerate}[topsep=1pt,itemsep=0pt,leftmargin=15pt]
\item $\onev_S$ can be perfectly recovered using either \eqref{eq:least_squares} and \eqref{eq:omega_min}.
\item $\onev_S$ is guaranteed to have minimum bandwidth among all indicator vectors satisfying the label constraints $\onev_S(L)$ on $\Xm_L$.
\end{enumerate}
\end{observation}
\noindent The observations above have significant implications: Given enough and appropriately chosen labeled data, BIG effectively recovers an indicator vector with minimum bandwidth, that respects the label constraints. 
Note that by labeling enough data appropriately, we mean to ensure that the cut-off frequency $\omega_L$ of the labeled set is greater than the bandwidth $\omega(\onev_S)$ of the indicator function of interest.
If this condition is not satisfied, both observations break down, i.e., the solutions of \eqref{eq:least_squares} and \eqref{eq:omega_min} would be different and serve only as approximations for $\onev_S$. 
Moreover, the minimum bandwidth signal $\fv_\text{min}$ satisfying the label constraints, would differ from $\onev_S$ and may not even be an indicator vector.
To help ensure that the condition is satisfied, one can use efficient optimal algorithms for labeling~\cite{gadde_14, shomorony_14}. We note that in practice, \eqref{eq:least_squares} can be solved via efficient iterative techniques \cite{narang_globalsip13}.
\section{Main Result}

We now consider the convergence of the bandwidth $\omega(\onev_S)$ of $\onev_S$, as the number of data points goes to infinity.
To simplify our analysis, we need certain assumptions: $p(\xv)$ must be Lipschitz continuous and twice differentiable on $\mathbb{R}^d$ and $\partial S$ must be smooth with radius of curvature $\tau > 0$.
Next, we note that the bandwidth of $\onev_S$, with respect to the Fourier basis specified by $\Lm$, can be written as~\cite{anis_14}
\begin{equation}
\omega(\onev_S) = \lim_{m \rightarrow \infty} \omega_m(\onev_S),
\end{equation}
where $\omega_m(\onev_S)$ is the $m^{\textrm{th}}$ order bandwidth estimate defined as:
\begin{equation}
\omega_m(\onev_S)=\left( \frac{\onev_S^T \Lm^m \onev_S}{\onev_S^T \onev_S} \right)^{1/m}.
\end{equation}
We now show that for the distance-based similarity graphs of~\eqref{eq:graph}, the bandwidth estimate 
converges to a function of $p(\xv)$, thus giving the connection between the BIG approach and the low density separation problem. Our result holds under the following set of conditions:
\begin{enumerate}[topsep=2pt,itemsep=0pt,leftmargin=15pt]
\item Large sample size: $n \rightarrow \infty$,
\item Shrinking neighborhood volume: $\sigma \rightarrow 0$,
\item Bandwidth estimate: $m\rightarrow \infty,\;
m/n \rightarrow 0,\; m\sigma^2 \rightarrow 0$,
\item $(1/\sigma)^{1/m} \rightarrow 1$,
\item $(n \sigma^{md+1})/(mC^m) \rightarrow \infty$, where $C = 2/(2\pi)^{d/2}$.
\end{enumerate}
\begin{theorem}
If conditions 1--5 hold, then 
\begin{equation}\label{eq:result}
\omega_m(\onev_S) \xrightarrow{\;\;\textit{p.}\;} \;\; \sup_{\sv \in \partial S} \; p(\sv),
\end{equation}
where ``\textit{p.}" denotes convergence in probability. Further, almost sure convergence holds if condition 5 is replaced by $\frac{n \sigma^{md+1}}{mC^m\log n}$ $\rightarrow \infty$.
\label{thm:main_thm}
\end{theorem}
Intuitively, the conditions 1--5 guarantee sparsity of the graph and govern the scaling of the bandwidth estimate order. The theorem essentially states that the estimate of the bandwidth of any indicator vector converges to the supremum of the underlying probability distribution on the corresponding decision boundary.
We now specify a graph construction scheme for which the result holds.
\begin{corollary} Equation~\eqref{eq:result} holds if for each value of $n$, we choose the parameters $\sigma$ and $m$ as follows
\begin{align}
\sigma &= n^{-x/(md+1)}, \quad 0 < x < 1, \\
m &= (\log{n})^y, \quad 1/2 < y < 1,
\end{align}
\end{corollary}
\noindent This result, along with the conclusions derived from the sampling theorem for graph signals in the previous section, forms the basis of justifying BIG as an effective method for SSL:
\emph{Given enough and appropriately chosen labeled data, BIG learns that decision boundary on which the supremum of the data density is minimum.}
Based on this, the following conclusions become apparent:
\begin{enumerate}[topsep=1pt,itemsep=0pt,leftmargin=15pt]
\item BIG is a variant of the constrained low density separation problem for finite number of data points, similar to other methods.
\item To learn a boundary that passes through a region of high probability density, more labeled data is required.
\end{enumerate}
\subsection{Proof sketch}
We now give an overview of the proof of Theorem \ref{thm:main_thm}. For our analysis, we consider the quantity $Y_m$ defined for $m \in \mathbb{Z}^+$ as:
\begin{equation}
Y_m = \frac{1}{\sigma} \left( \frac{\onev_S^T \Lm^m \onev_S}{\onev_S^T \onev_S} \right).
\end{equation}
We prove the following convergence result:
\begin{align}
\left( Y_m \right)^{1/m} &\xrightarrow{\;\;\textit{p.}\;} \left(\E{ Y_m }\right)^{1/m} \longrightarrow \;\; \sup_{\sv \in \partial S} \; p(\sv),
\label{eq:stoc_conv}
\end{align} 
where the second arrow denotes sure (deterministic) convergence. Since $(1/\sigma)^{1/m}$ $\rightarrow 1$ (condition 4), we can reach the desired result of~\eqref{eq:result} from~\eqref{eq:stoc_conv} through a simple argument.
Before providing a sketch of the proof for~\eqref{eq:stoc_conv}, we first discuss how they rely on the conditions in the Theorem's statement. 
Conditions 1 and 5 are required to ensure stochastic convergence of the left hand side of \eqref{eq:stoc_conv}.
Conditions 2 and 3 are required to show sure convergence of the right hand side of \eqref{eq:stoc_conv}. 
The proof of~\eqref{eq:stoc_conv} begins by re-expressing $Y_m$ as $\frac{\frac{1}{n\sigma} \onev_S^T \Lm^m \onev_S}{\frac{1}{n} \onev_S^T \onev_S}$, and studying the convergence of the numerator and denominator separately. By the strong law of large numbers, we conclude that
\begin{equation}
\frac{1}{n} \onev_S^T \onev_S \; \xrightarrow[]{a.s.} \; \int_S p(\xv)d\xv.
\end{equation}
For the numerator, we decompose it into two parts -- a variance term for which we show stochastic convergence and a bias term for which we prove deterministic convergence. Let $V = \frac{1}{n\sigma} \onev_S^T \Lm^m \onev_S$, then we have the following results for $V$ and $\E{V}$:
\begin{lemma}[Concentration]
For every $\epsilon > 0$, we have:
\begin{align}
&\Pr{\left( \left| V - \E{V } \right| > \epsilon \right)} \nonumber \\
& \quad \quad \leq 2 \exp{ \left( \frac{-[n/(m+1)]\sigma^{md+1}\epsilon^2}{2C^m \E{V} + \frac{2}{3}\left| C^m - \sigma^{md+1} \E{V}\right|\epsilon} \right) },
\label{eq:v_conc}
\end{align}
where $C = 2/(2\pi)^{d/2}$. Note that the right hand side goes to $0$ when condition 5 holds.
\end{lemma}
\begin{proof}[Proof sketch]
We begin by expanding $V$ as follows:
\begin{align}
V &= \frac{1}{n\sigma} \onev_S^T ( \Dm - \Wm)^m \onev_S \\
%
%
&= \frac{1}{n^{m+1}} \sum_{i_1,i_2,\dots,i_{m+1}} g\left( \Xm_{i_1},\Xm_{i_2},\dots,\Xm_{i_{m+1}} \right).
\end{align}
The above expansion has the form of a V-statistic. Recalling that $w_{i,j} = K(\Xm_i,\Xm_j)$, we note that $g$ is composed of a sum of $2^m$ terms, each a product of $m$ kernel functions. Therefore,
\begin{equation}
g \leq \frac{1}{\sigma}2^m\|K\|_\infty^m = \frac{1}{\sigma}\left(\frac{2}{(2\pi\sigma^2)^{d/2}}\right)^m = \frac{C^m}{\sigma^{md+1}}.
\label{eq:g_bound}
\end{equation}
In order to apply a concentration inequality for V, we first re-write it in the form of a U-statistic by regrouping terms in the summation so that repeated indices are removed, as given in~\cite{hoeffding_63}:
\begin{align}
V &= \frac{1}{n^{(m+1)}} \sum_{(n,m+1)} g^*\left( \Xm_{i_1},\Xm_{i_2},\dots,\Xm_{i_{m+1}} \right),
\end{align}
where $\sum_{(n,m+1)}$ denotes summation over all (m+1)-tuples of distinct indices taken from the set $\{1,\dots,n\}$, $n^{(m+1)} = n.(n-1)\dots(n-m)$ is the number of (m+1)-permutations of $n$ and $g^*$ is a convex combination of certain values of $g$ that absorbs repeating indices satisyfing the property:
\begin{align}
g^*\left( \xv_1,\xv_2,\dots,\xv_{m+1} \right) &= \frac{n^{(m+1)}}{n^{m+1}} g\left( \xv_1,\xv_2,\dots,\xv_{m+1} \right) \\
&+ \Order{\frac{m}{n}} \text{(terms with repeated indices)}. \nonumber
\end{align}
Therefore, $g^*$ has the same upper bound as that of $g$ derived in \eqref{eq:g_bound}. Moreover, using the fact that $\E{V} = \E{g^*}$, we can bound the variance of $g^*$ as
\begin{equation}
\Var{g^*} \leq \E{(g^*)^2}  \leq \|g^*\|_\infty \E{g^*} = \frac{C^m}{\sigma^{md+1}} \E{V}.
\end{equation}
Finally, plugging in the bound and variance of $g^*$ in Bernstein's inequality for U-statistics~\cite{hoeffding_63,hein_thesis_06}, we arrive at the result of \eqref{eq:v_conc}.
\end{proof}
\begin{lemma}[convergence of bias] 
As $n \rightarrow \infty$, $\sigma \rightarrow 0$ and $m\sigma^2 \rightarrow 0$, we have
\begin{equation}
\E{V} \longrightarrow \frac{t(m)}{\sqrt{2\pi}} \int_{\partial S} p^{m+1}(\sv)d\sv,
\end{equation}
where $t(m) = \sum_{r = 1}^{m-1} {m-1 \choose r} (-1)^r (\sqrt{r+1} - \sqrt{r})$.
\end{lemma}
\begin{proof}[Proof sketch]
We use the following properties of $K_{\sigma^2}(\xv,\yv)$:
\begin{align}
\int K_{\sigma^2}(\xv,\yv)p(\yv)d\yv &= p(\xv) + \Order{\sigma^2}, 
\label{eq:K_int1}
\\
\int K_{a\sigma^2}(\xv,\zv) K_{b\sigma^2}(\zv,\yv) p(\zv) d\zv &= K_{(a+b)\sigma^2}(\xv,\yv) \; p\left( \frac{b \xv + a \yv}{a + b} \right) \nonumber \\
&\quad+ \Order{\sigma^2}.
\label{eq:K_int2}
\end{align}
We evaluate $\E{V}$ term by term by writing $\Lm^m = (\Dm-\Wm)^{m-1}(\Dm-\Wm)$. For all terms in the expansion of $(\Dm-\Wm)^{m-1}$ containing $r$ occurrences of $\Wm$, we use \eqref{eq:K_int1} and \eqref{eq:K_int2} and $m\sigma^2 \rightarrow 0$ to get
\begin{align}
&\E{\frac{1}{n\sigma} \yv^T [\Dm^{m-1-r},\Wm^r] (\Dm-\Wm) \yv} \nonumber \\
&= \frac{1}{\sigma} \int_S \int_S K_{r\sigma^2}(\xv,\yv) p^\alpha(\xv) p^\beta(\yv) d\xv d\yv \nonumber \\
&\quad- \frac{1}{\sigma} \int_S \int_S K_{(r+1)\sigma^2}(\xv,\yv) p^{\alpha'}(\xv) p^{\beta'}(\yv) d\xv d\yv + O(\sigma),
\label{eq:rth_term}
\end{align}
where $\alpha + \beta = m+1$ and $\alpha' + \beta' = m+1$. It can be shown that the right hand side of \eqref{eq:rth_term} converges to $\frac{\sqrt{r+1}-\sqrt{r}}{\sqrt{2\pi}}\int_{\partial S} p^{m+1}(\sv) d\sv$. Putting everything together, we get the desired result.
\end{proof}
\noindent Finally, we note that as $m \rightarrow \infty$, we have
\begin{equation}
\left( \frac{\frac{t(m)}{\sqrt{2\pi}} \int_{\partial S} p^{m+1}(\sv)d\sv}{ \int_S p(\xv)d\xv} \right)^{1/m} \longrightarrow \sup_{\sv \in \partial S} p(\sv).
\end{equation}
\section{Experimental results}
In this section, we numerically analyze our asympotic results and show that they are also useful in practice. For our experiments, we considered a 2-D Gaussian mixture model with three Guassians: $\mu_1 = [-2, 0], \Sigma_1 = 0.64\Id$, $\mu_2 = [0, 0], \Sigma_2 = 0.25\Id$ and $\mu_3 = [2, 0], \Sigma_3 = 0.16 \Id$, with corresponding mixture proportions: $\alpha_1 = 0.5, \alpha_2 = 0.2, \alpha_3 = 0.3$. 
The plot of the density is given in Figure~\ref{fig:pdf_cut}. For computing edge weights of the graph, we set $\sigma = 0.1$. 
\par
In our first experiment, we studied the behavior of the empirical bandwidth estimate $\omega_m(\onev_S)$ with $n$ for different values of $m$. We used sample sizes varying from $n = 500$ to $n = 2500$, drawn \emph{i.i.d.} from the pdf, to compute $\omega_m(\onev_S)$ with $m=10,20,30$ for the 2D hyperplane $\partial S: x = 0$. This experiment was repeated 100 times and the mean was compared with the supremum of the boundary (Figure~\ref{fig:cut_vs_m}). We observe that as $m$ increases, the mean empirical bandwidth estimate approaches the theoretical limit (for a fixed $m$, the mean value decreases slightly with $n$ since for a higher $n$, the rate of convergence of $\omega_m(\onev_\Sc)$ with $m$ is slower). Further, as $n$ increases, the standard deviation of the empirical bandwidth decreases, indicating asymptotic convergence of the empirical quantity.
\par
Next, we validate the result of Theorem~\ref{thm:main_thm} for different boundaries.
This is carried out as follows: we fix the bandwidth approximation factor to $m = 20$ and compare $\omega_m(\onev_S)$ with $\sup_{\sv \in \partial S} \; p(\sv)$, for different positions of the boundary $\partial S: x = c$ (obtained by sweeping $c$ as shown in Figure~\ref{fig:pdf_cut}). 
This procedure is carried out 100 times and the results are shown in Figure~\ref{fig:cut_results}. We observe that the empirical and the limit values are fairly close to the supremum of $p(\xv)$ over the boundary, the slight gap arises due to finite $m$. The overshoot of the empirical quantity over the supremum for some positions of the boundary happens because $\sigma$ is not small enough for convergence of the bias term at those parameter settings.
\vspace{-0.3cm}
\begin{figure}[htb]
\centering
\includegraphics[width=6.5cm]{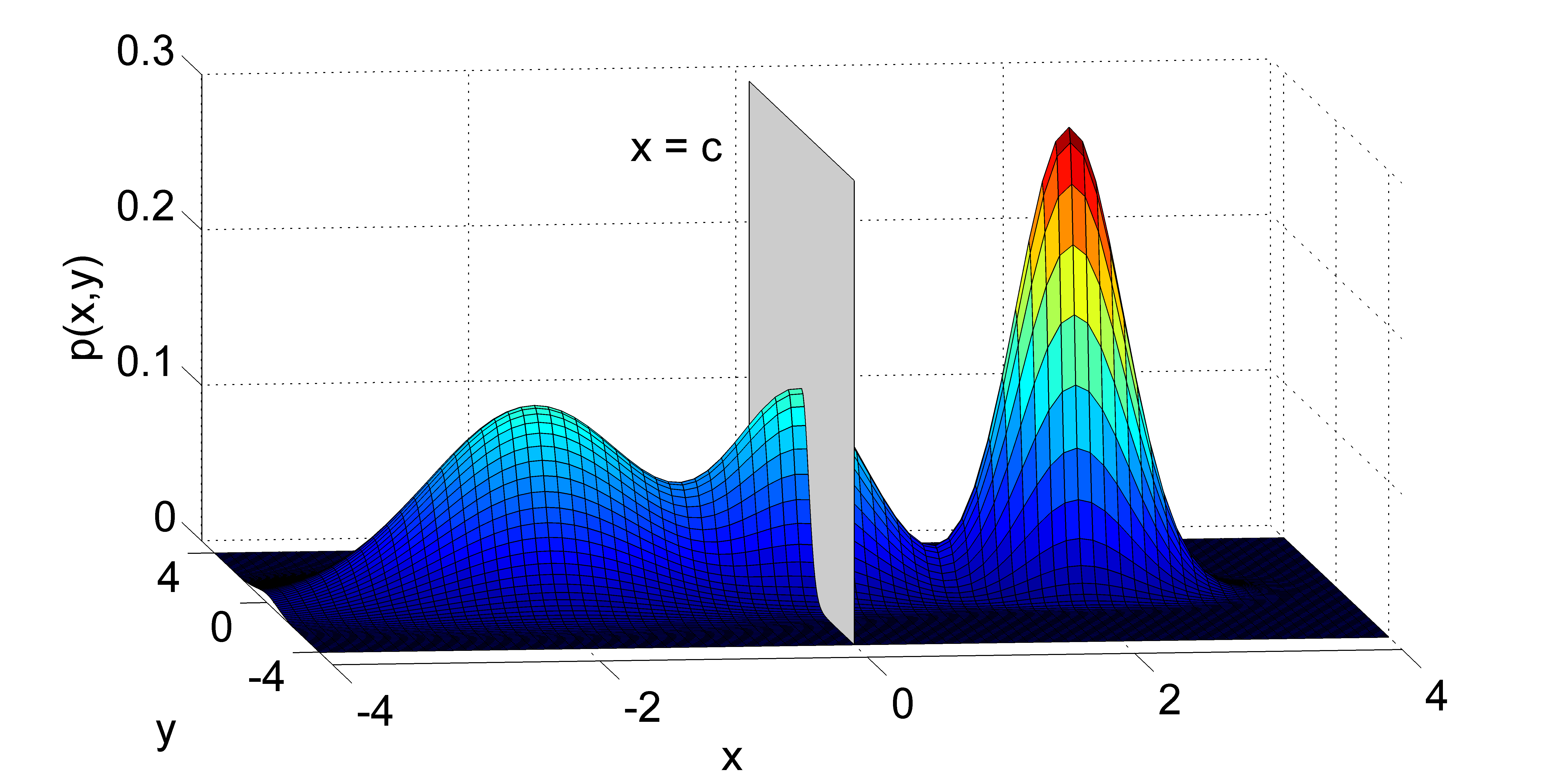}
\caption{2D GMM used in experiments. Family of hyperplanes $x = c$ that cut perpendicular to the first dimension (the ``informative" dimension for the pdf) are taken as decision boundaries $\partial S$. }
\label{fig:pdf_cut}
\end{figure}
\vspace{-0.3cm}
\begin{figure}[htb]
\centering
\includegraphics[width=8cm]{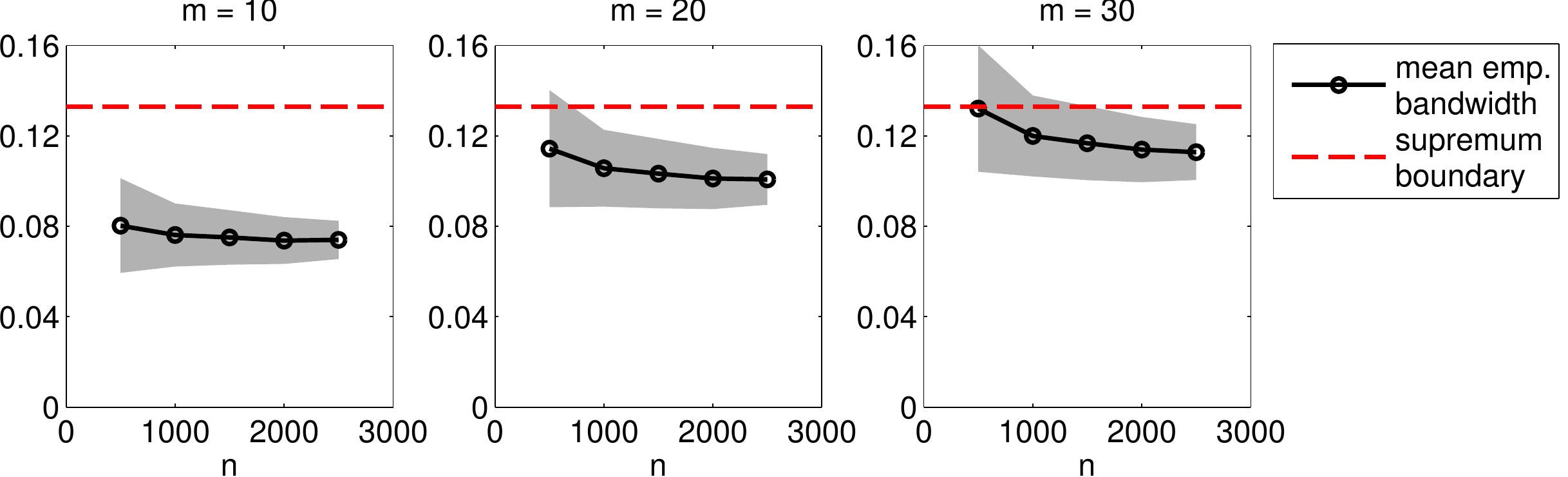}
\caption{Convergence of $\omega_m(\onev_S)$ with $n$ for the boundary $\partial S: x = 0$ 
and different $m$.
$\sigma$ is fixed at $0.1$. Shaded area indicates standard deviation over 100 experiments. Red-dashed line shows $\sup_{\sv \in S} p(\sv)$.}
\label{fig:cut_vs_m}
\end{figure}
\vspace{-0.4cm}
\begin{figure}[htb]
\centering
\includegraphics[width=6cm]{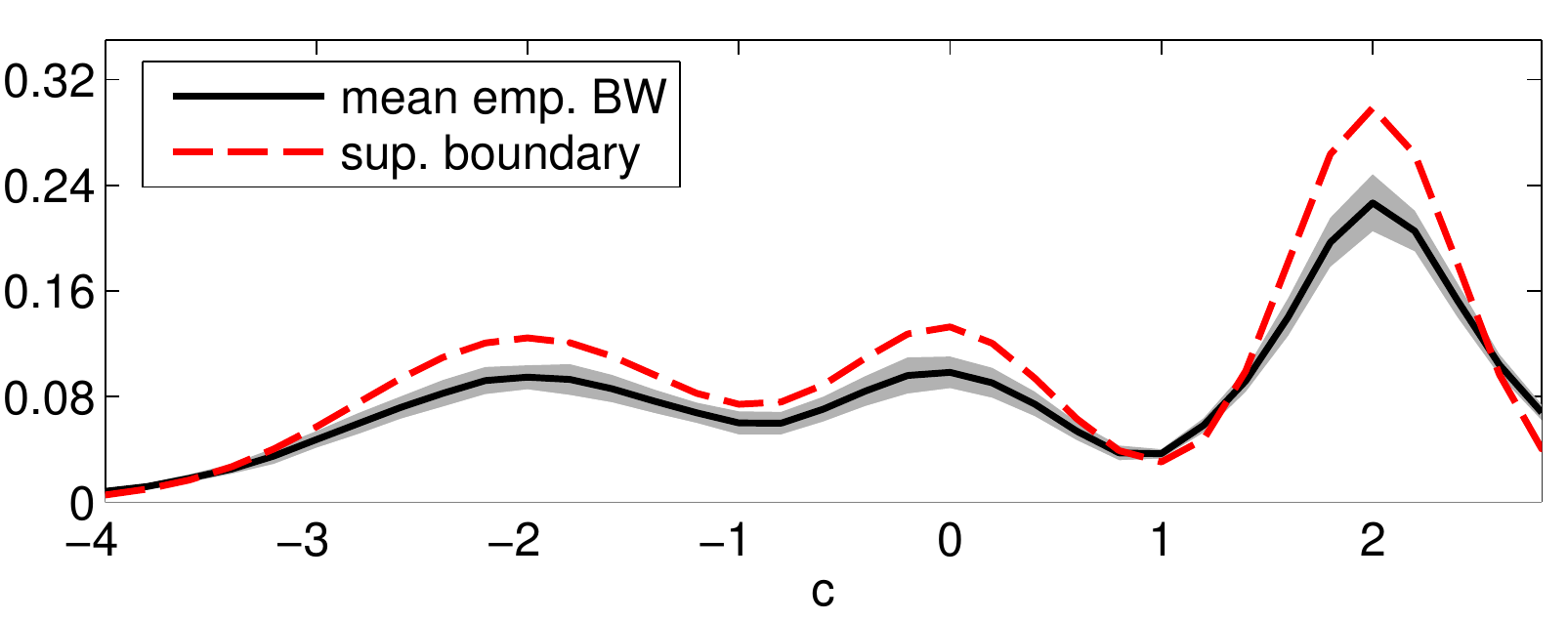}
\caption{Convergence of $\omega_m(\onev_S)$ with $m = 20$ for varying hyperplane parameter $c$. $n$ and $\sigma$ are fixed at 2500 and 0.1. Shaded area indicates standard deviation over 100 experiments. Red-dashed line shows $\sup_{\sv \in S} p(\sv)$.}
\label{fig:cut_results}
\end{figure}
\vspace{-0.5cm}
\section{Summary}
In this paper, we provided an asymptotic justification of using the bandlimited interpolation of graph signals (BIG) approach for semi-supervised learning (SSL).
We considered a statistical setting and computed the limiting value of the bandwidth estimate for any indicator signal defined on a distance-based similarity graph that is fairly common in practice.
As a consequence of our result and the sampling theory for graph signals, the BIG approach for SSL is found to be closely related to the low density separation problem. 
We show through experimental analysis that the theoretical results are useful in practical scenarios.
In future work, we aim to exploit this result for finding the label complexity of any indicator signal in the ``BIG for SSL" framework, and comparing the BIG approach with existing methods, to further understand the value of labeled data.

\bibliographystyle{IEEEbib}
\bibliography{paper_v11}

\end{document}